\newif\ifdraft
\draftfalse
\newif\ifaftersub
\aftersubfalse

\documentclass[sigconf, nonacm]{aamas}

\usepackage{balance} 

\usepackage{amsmath,amsfonts,bm}

\def\eqref#1{equation~\ref{#1}}

\def\1{\bm{1}}

\def\eps{{\epsilon}}

\DeclareMathAlphabet{\mathsfit}{\encodingdefault}{\sfdefault}{m}{sl}
\SetMathAlphabet{\mathsfit}{bold}{\encodingdefault}{\sfdefault}{bx}{n}

\newcommand{\E}{\mathbb{E}}

\newcommand{\sigmoid}{\sigma}

\newcommand{\Var}{\mathrm{Var}}

\usepackage[english]{babel}
\usepackage{algorithm}
\usepackage{algorithmic}

\usepackage{amsmath}
\usepackage{graphicx}
\usepackage{xcolor}
\usepackage{xspace}
\usepackage{multirow}
\usepackage[nameinlink]{cleveref}
\usepackage{placeins}
\usepackage{amsthm}
\newtheorem{theorem}{Theorem}[section]

\newtheorem{lemma}[theorem]{Lemma}
\newtheorem{proposition}[theorem]{Proposition}

\theoremstyle{definition}
\newtheorem{definition}{Definition}[section]

\renewcommand\E[1]{\mathbb{E}\left[\,#1\,\right]}

\renewcommand\sigmoid{\text{sigmoid}\xspace}

\newcommand\relu[1]{ReLU\left(\,#1\,\right)}

\newcommand\contmodel{\operatorname{Continuous Model}\xspace}
\newcommand\binmodel{\textsc{Federated\ Zampling}\xspace}
\newcommand\localmodel{\textsc{Local\ Zampling}\xspace}

\newcommand{\smallarch}{\textsc{small architecture }}
\newcommand{\mnistfc}{\textsc{mnistfc }}

\title[Trading-off Accuracy and Communication Cost in Federated Learning]{Trading-off Accuracy and Communication Cost in Federated Learning}
\author{Mattia Jacopo Villani}
\affiliation{
  \institution{King's College London}
  \city{London}
  \country{United Kingdom}}
\email{mattia.villani@kcl.ac.uk}

\author{Emanuele Natale}
\affiliation{
  \institution{Université Côte d’Azur}
  \city{Sophia-Antipolis}
  \country{France}}
\email{emanuele.natale@univ-cotedazur.fr}

\author{Frederik Mallmann-Trenn}
\affiliation{
  \institution{King's College London}
  \city{London}
  \country{United Kingdom}}
\email{frederik.mallmann-trenn@kcl.ac.uk}
\date{}
\begin{document}

\begin{abstract}
    Leveraging the training-by-pruning paradigm introduced by Zhou et al. [NeurIPS'19], Isik et al. [ICLR'23] introduced a federated learning protocol that achieves a 34-fold reduction in communication cost.
    We achieve a compression improvements of orders of orders of magnitude over the state-of-the-art.
    The central idea of our framework is to encode the network weights $\vec w$ by a the vector of trainable parameters $\vec p$, such that $\vec w = Q\cdot \vec p$ where $Q$ is a carefully-generate sparse random matrix (that remains fixed throughout training).
    In such framework, the previous work of Zhou et al. [NeurIPS'19] is retrieved when $Q$ is diagonal and $\vec p$ has the same dimension of $\vec w$.
    
    We instead show that $\vec p$ can effectively be chosen much smaller than $\vec w$, 
    while retaining the same accuracy at the price of a decrease of the sparsity of $Q$. 
    Since server and clients  only need to share $\vec p$, such a trade-off leads to a substantial improvement in communication cost.
    Moreover, we provide theoretical insight into our framework and establish a novel link between training-by-sampling and random convex geometry.

    
\end{abstract}

\maketitle



\section{Introduction}

In many scenarios, training cannot be done on a single machine, either due to the large size of the dataset or privacy concerns that prevent direct data sharing (e.g., hospital data). To enable efficient and secure training on mobile devices, \emph{federated learning} was introduced \cite{firstfed,fedfirst2}. In this approach, multiple agents or clients train on separate partitions of the data, periodically sharing learned parameters with a central server. The server aggregates these parameters and distributes the updated model back to all clients.

In this server-client framework 
the training data is split across many clients who, with the help of a central server, jointly train a model without ever sharing data directly. 
A well-known example is \cite{keyboard}, which demonstrates that federated learning for next-word prediction on virtual keyboards significantly improved recall compared to centralized training on lower-quality data. Furthermore, the distributed computation across clients can accelerate training.

However, the main issue of this framework is that it requires the model (i.e., all learnable parameters) to be frequently shared with all clients resulting in a very high communication cost. Communicating large amounts of data slows down training significantly.   
To address this issue \cite{isik} proposed a novel framework that allows the clients to simply share one bit (rather than a float) per parameter of the model. This results in a 32-fold reduction in communication cost.\footnote{Their model still requires the server to share the a float for every parameter.}
Their work builds on the training-by-pruning paradigm introduced by \cite{zhou}. 
In this paradigm, instead of applying gradient descent on a fixed collection of weights, weights are randomly initialised and removed with some probability that changes adaptively. 
Specifically, each randomly initialised parameter (weight) $w_i$ has an associated probability $p_i$ of being used in the neural network architecture. 
During training, we sample every weight, and with probability $1-p_i$, the weight gets a value of 0. 
The goal then is to use gradient-based methods to train the distribution of $\vec{p}$ for each $p_i$.\footnote{Technically, one trains via SGD a score value $s_i$ from when then the $p_i$ is calculated: $p_i\sigmoid(s_i)$ to ensure that the values of $p_i$ are in $[0,1]$.} so that the sampled network has minimal loss. 
Surprisingly, this seems to work very well in practice, leading to the development of other variants, such as, the Edge-PopUp algorithm by \cite{ramanujan}. 

In federated learning, transmitting binary masks instead of exact parameter values not only reduces communication cost but also enhances privacy. Additionally, the sparse network architecture lowers inference costs. Using this technique, \cite{isik} achieved high accuracy in training artificial neural networks (ANNs) with a 32-fold communication reduction. They further applied compression techniques that capitalize on patterns of consecutive 1s or 0s, yielding a total communication reduction of 33-34 times.

Our paper extends far beyond this 34-fold reduction in communication cost, achieving a 1024-fold total reduction. 
The idea is that the clients and server jointly (from a shared random seed) initialise a sparse \emph{coefficient  matrix} $Q$ where each row has $d$-non zero parameters. 

The trainable parameters are given by a $n$-dimensional probability vector $\vec p$ with $n \ll m$, where $m$ is the total number of parameters of the model. 
Each  client $k$ samples a binary vector $z^{(k)} \sim Bern(\vec p)$, where $Bern(\vec p)$ is the component-wise Bernoulli distribution.
Each client then
and calculates its weights as follows
$\vec w^{(k)} = Qz^{(k)}$.\footnote{Note that this model is different from parameter sharing as potentially all weights/parameters can be different. 
There is some correlation between each parameter, but ultimately it fairly weak and results in high accuracies as we show in our experiments (see \Cref{sec:experiments}).
}
After calculating the local gradient, each client then updates $\vec p$ to $\vec p_{new}$ and samples once more $z_{new}^{(k)} \sim Bern(\vec p_{new})$ before sending it to the server.
Crucially, sending $z_{new}^{(k)}$ takes at most $n$ bits
which can be orders of magnitude smaller than $m$ (see \cite{isik})\footnote{As mentioned before, they also use some other forms of compression when the binary vector has many $0$s or $1s$, but the savings are negligible in the grand scheme of things.} and much smaller than sending float values to the server which would cost $32m$. 
Fascinatingly, when sending $n$ bits vs $32m$ bits we only see (\Cref{fig:isik}) a very small decrease in accuracy up to a certain threshold of $n$. 
When $n$ becomes too small, the model starts to deteriorate (\Cref{fig:compression}). However, for example in the case of MNIST, $n$ can be smaller than $m$ by a factor of $32$ and we only lose a $7$ percent points of accuracy - giving a total compression of $32 \cdot 32\approx 1024$. 

In some sense, the framework of \cite{isik} (building on training-by-pruning from \cite{zhou}) is a special case of our model when $d=1$ and $n=m$.\footnote{There are some minor difference, technically $Q$ would be a diagonal matrix and the scores are calculated using the sigmoid function rather than our clipping, but overall it is virtually the same. Crucially, since $n=m$ no further compression can be obtained. } It turns out that allowing larger values of $d$ increases accuracy. In \Cref{sec:d} we discuss the reasons for this. Typically, in federeated learning, the communication complexity is measure by the communication each clients needs to send, but it's easy to see that a low server communication cost can also be crucial. In \cite{isik} and our main algorithm, the server still broadcasts the entire model. However, in our approach that cost is only $32n$ vs $32m$, which, as argued above, can be smaller by an order of magnitude without losing too much accuracy.

Although we frame our results in the setting of federated learning, the model itself is also interesting to understand the impact dimension reduction w.r.t. the learnable parameters. For this reason we also introduce the $\localmodel$ algorithm  in \Cref{sec:model} a simplified local version of the algorithm outside of the context of federated learning. Here, we study the generalisation of the learned vector $\vec p^*$.
Finally, we complement our experiments with theoretical results on the rich convex random geometry (namely zonotopes) induced by our framework.


\subsection{Our Contribution}

We present \textsc{Zampling} (Zonotope Sampling), a new training-by-sampling framework inspired by convex random geometry that achieves small reductions in accuracy for state of the art factors of compression in communication cost in the federated learning setting.
Given a an arbitrary neural network architecture, \textsc{Zampling} replaces the model's parameters $\vec{w}$ with a product of a probability vector $\vec{p}$ and a sparse influence matrix $Q$, enabling both training-by-sampling for any model and state of the art compression in parameter communication costs in the federated setting. 
Our framework enables large parameter compression that lowers the communication cost by several orders of magnitude in the context of Federated Learning: \textsc{Federated Zampling}. 
Moreover, our work is a generalisation of Zhou et. al. \cite{zhou} and sheds novel theoretical insight on the generalisation benefits from training-by-sampling. 
%
We now give an overview of our contributions.

\textbf{Highly Compressed Communication Cost in Federated Learning} Our main result
shows that we can reduce the client communication cost by a factor of 1024 in comparison to the naïve algorithm (and a factor of 32 w.r.t. the state-of-the-art)
 while witnessing only a 3\% point reduction in accuracy. 
We also  provide theoretical insights into why federated learning is helpful in our setting (\Cref{sec:fedswithbenefits}).
Moreover, we achieve a fine-grained Trade-off between accuracy and compression for $\localmodel$. 
In \Cref{sec:compression_Exp} we perform numerical experiments on a small architecture (to avoid redundancy in parameters). 
We show a relatively smooth trade-off between accuracy and compression factor. Our results also show that increasing $d$ slightly benefits higher accuracy, but large values of $d$ yield no additional benefits. 


\textbf{Theoretical Foundations of Training-by-Sampling and a Novel Link to Random Convex Geometry}. 
We  provide theoretical evidence on the benefits of deploying training-by-sampling algorithms in the federated setting (see \Cref{sec:theory}). 
First we show what the impact of $d$ is. We argue in \Cref{pro:cherrypickD} that the expected maximum norm of $Q_i \vec p$ scales as $\Theta(\sqrt{d/n_\ell})$ where the maximum is taken over $\vec p$. This implies that for larger values of $d$ the model can reach larger weights. We also show characterize the number of $0$ entries in $\vec w(0)$ and $\vec 0$ columns in $Q$ that can arise for small $d$ and may limit the expressivity of the model.
See \Cref{sec:d} for the details.
In \Cref{sec:convex} we develop a novel link with random convex geometry that allows us to prove statements about the generalisation capabilities of training-by-sampling. 
This allows us then to prove the benefits on federated learning has when it comes to learning by sampling (see \Cref{sec:fedswithbenefits}).

Through our formalism (\Cref{sec:convex}), and experimentally (\Cref{sec:sensitivity_exp}), we provide evidence that training-by-sampling improves generalisation.
This is because our method can be described as searching for good solutions on the vertices of a convex shape known as a zonotope. 
Through training, the algorithm finds a sub-zonotope where performance is good. 
Even though our training method only measures performance at the vertices of this shape, we show experimentally that good performance is maintained in all the region characterised by the shape, comparing it to gradient-based training \Cref{sec:sensitivity_exp}. 
Finally, our findings (see \Cref{fig:integralitygap}) show that when training the sampled network instead of simply training  $\vec p$ directly without ever sampling, then the network does not generalize as well and suffers a huge drop of accuracy, to which we refer as the \emph{integrality gap}.

\subsection{Related Work}

There is long line of research on communication efficient federated learning. The authors of \citep{aji2017sparse} achieve a reduction in communication cost by only exchanging the largest gradients. In a similar spirit, 
\cite{barnes2020rtop} also estimates the SGD update and \cite{lin2017deep} sparsifies the gradients to optimize for bandwidth usage.
\cite{isik2022information} present an algorithm and theoretical foundations on sparsificaiton using successive refinement.
Other methods that build on low-rank approximation include \cite{wang2018atomo} and \cite{basat2022quic}.
Our work extends this by showing that an even smaller set of parameters can be shared while maintaining accuracy.

Other approaches to reducing communication costs include parameter sharing schemes like those proposed by ~\cite{chen2015compressing} and ~\cite{barnes2020rtop}, which aim to minimize communication through compressed updates or estimating SGD steps. Similarly, ~\cite{theis2022lossy} explored the use of shared randomness to achieve lossless communication compression. 
Our approach draws on similar principles but extends these ideas by combining the use of sparse random matrices and convex geometry insights.

Up to minor difference, that the special case $n=m$ and $d=1$ is akin to the training-by-pruning algorithm considered in \cite{zhou}. 
Said algorithm motivated the Strong Lottery Ticket Hypothesis \cite{pensiaOptimalLotteryTickets2020}, and other training-by-pruning algorithms such as Edge-PopUp \cite{ramanujanWhatHiddenRandomly2020}; finding lottery tickets has been shown to be computational difficult
\cite{gadhikarWhyRandomPruning2023,otsukaPartialSearchFrozen2024}.
It also found applications in federated learning for distributed training of neural networks with reduced communication cost \cite{isik}.

Recent papers demonstrate a strong focus on advancing federated learning techniques, particularly in contexts involving data heterogeneity and decentralized environments. 
For example, the AAMAS papers (2023-2024) ~\cite{aamas2023} introduces attention mechanisms to address the aggregation challenges inherent in federated reinforcement learning across heterogeneous agents.
Similarly, the study~\cite{aamas2024} explores methods to handle non-i.i.d. data by clustering clients with similar distributions prior to aggregation, thereby improving model accuracy and robustness. 

Finally, if side information is available, the authors of \cite{isikAdaptiveCompressionFederated2024} use sampling to achieve a 2650-fold  compression in the context of federated learning.
Moreover, shared randomness has been successfully used to reduce the communication cost, e.g., \cite{canonne2015communication,kurri2021coordination,isik2024exact}.
Our method achieves a 1024-fold reduction in communication cost. 

\subsection{Our Framework: \textsc{Zampling}}\label{sec:model}

We consider the setting of federated learning, where a server and $
K$ clients jointly train a neural network model. 
For any neural network architecture, let $m$ be the total number of parameters of the model and let $n$ be a number of \emph{trainable parameters} with $n\leq m$.  Let $\vec p(t)  \in [0,1]^n$ be the vector of these parameter at time $t$ and we refer to it as the \emph{probability distribution vector}. 

Let $Q=(q_{i,j})_{i\leq m, j\leq n}$ be a randomly initialised but non-trainable matrix $\in \mathbb{R}^{m\times n}$ that describes how each trainable parameter (in $\vec p$) affects each weight, i.e., $q_{i,j}$ describes how the $j$th trainable parameter influences the $i$th weight. Let $d$ the weight degree (each weight is influenced by $d$ trainable parameters), i.e., the number of non-zero entries per row. The matrix $Q$ does not change over time and will never be sent --- we assume that server and clients both have $Q$ which can be realised by sharing the same random seed to generate identical matrices. We assume that the data is distributed IID among the clients.

\paragraph{Further Notation}
Let $f(x)=\max(\min(x,1),0)$ be the ReLU function clipped at $1$.
Let $D_{(k)}$ be the dataset at agent $k \leq K, k \in \mathbb{N}$ and $D=\bigcup_{k \leq K} D_{(k)}$.
Given a weight vector $\vec w$, we use \( g_{\vec w}: X\rightarrow Y \) to describe the resulting network (note that server and clients use the same architecture and hence the weight vector fully determines the model).
We define the \emph{compression factor} to be $m/n$. In our terminology each \emph{round} has  up to 100 (training) epochs.  Clients and server exchange messages at the beginning and end of each round. 

\paragraph{Initialization}
We generate a coefficient matrix $Q\in \mathbb{R}^{m \times n}$. For each $i\leq m$ sample a set of $d$ indices $\mathcal{I}_{i} \in [n]^d$ without replacement. 
Then generate $Q$ as follows.
\begin{equation}
q_{i,j}  \sim
\begin{cases} N(0, \sigma_i^2) & \text{ if $j \in \mathcal{I}_i$} \\
0 & \text{ otherwise}
\end{cases},
\end{equation}
where $\sigma_i^2 = \frac{6}{d n_\ell}$ and $n_\ell$ is the fan-in (number of incoming weights) of the target neuron associated to weight $w_i$ (see \Cref{lem:init} where we show that this results in a Kaiming-He initialization).

The initial values of $\vec p$ are drawn from an $n$-dimensional uniform distribution $\vec p(0)\sim U(0,1)^n $.
The initial values of the weights are now calculated by setting $\vec w_{init} = \vec w(0) = Q \vec p(0)$.

\paragraph{$\localmodel$ (Centralized Version)}
We start by describing the $\localmodel$ framework/algorithm, which trains the model locally on one central machine (not in the federated learning framework). 
In each round $t$:
\begin{itemize}  
 \item Set
$\vec s(t) =  \vec p(t)$
    \item Sample the binary mask $\vec z(t) \in \{ 0,1\}^n$: $ \vec z(t) \sim \text{Bernoulli}(\vec p(t))$ 
    \item Calculate the weights vector\footnote{We assume that clients and server have the same mapping of entries in this vector to weights in the architecture.} $\vec w(t) = Q \cdot \vec z(t)$.

\item $\vec w(t)$ is used for the forward pass, and the loss $\mathcal{L}(g_{w(t)}, D)$  is backpropagated to update the score mask  as $\vec s(t+1) = \vec s(t) - \eta \nabla \mathcal{L}(g_{w(t)}, D)$ ($\eta$ is the  learning rate).
\end{itemize}

Note that this training method ensures that the sampled networks achieve good performance, as we discuss in \Cref{sec:experiments}; this is not the case if we train the same architecture traditionally (without sampling throughout training), as we discuss in \Cref{app:initialisation}, where we relate the sampled versus non-sampled gap in the network's performance to the initialisation of $\vec{p}$. 

\begin{figure}[ht!]
\includegraphics[width=0.35\textwidth]{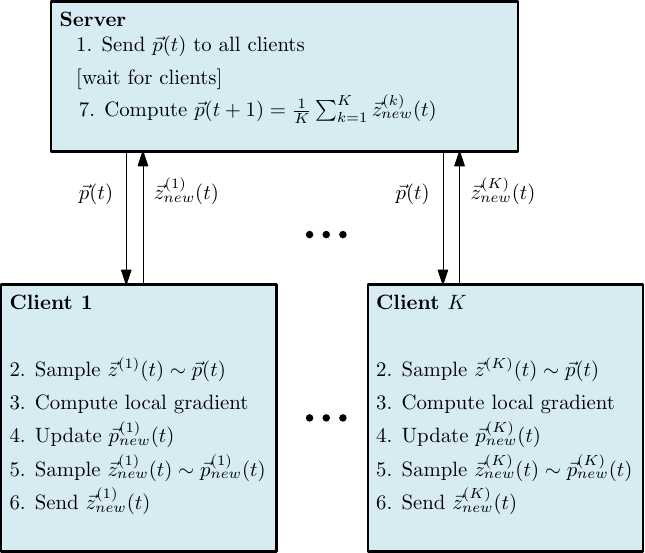}\Description{An illustration of the FedZample algorithm.}
\label{fig:binmodel}
\caption{An illustration of the \binmodel algorithm.}
\end{figure}

\paragraph{$\binmodel$ (Federated Learning Version)} 
We now describe the $\binmodel$ (see \Cref{fig:binmodel} for an illustration of the protocol). For the initialisation, it is important that both server and clients have $Q$. This can be realised by building $Q$ through using a shared random number generator.  
In each round $t$:

\begin{itemize}  
    \item  The server sends $\vec p(t)$ to all its clients (as a float).

    \item Each client $k$ calculates
$\vec s^{(k)}(t) = \vec p{(t)}$
    \item  Each client $k$ samples the binary mask $\vec z^{(k)}(t) \in \{ 0,1\}^n$ as follows $ \vec z(t) \sim \text{Bernoulli}(\vec p(t))$ 
    \item Each client  calculates it's local weights as a vector\\ $\vec w^{(k)}(t)= Q \cdot \vec z^{(k)}(t) $

\item $\vec w{(t)}$ is used for the forward pass, and the loss $\mathcal{L}(g_{w^{(k)}(t)}, D_k)$ on the local task is backpropagated
to update the score mask at client $k$ as as  $\vec s_{new}^{(k)}(t) = \vec s^{(k)}(t) - \eta \nabla \mathcal{L}(g_{w^{(k)}{(t)}}, D_k)$ ($\eta$ is the local learning rate).

\item  Each client $k$  converts the score back to probabilities: \\
$\vec p_{new}^{(k)}(t)= f\left(\vec s_{new}^{(k)}(t) \right)$
\item  Each client  samples from this $\vec z_{new}^{(k)}(t)\sim \text{Bernoulli}(\vec p_{new}^{(k)}(t))$ 
and sends $\vec z_{new}^{(k)}(t)$ to the server.
\item The server  calculates $\vec  p(t+1)= \frac{1}{K} \sum_{i=k}^K \vec z^{(k)}(t)$.
\end{itemize}

Sometimes we will also compare to the $\contmodel$, where no sampling happens, i.e., $\vec w^{(k)}(t)=Q \vec p(t)$ instead of $\vec w^{(k)}(t)=Q \vec z^{(k)}(t)$. The rest is exactly the same - including how the gradients are updated, i.e., 
$\nabla_{\vec{s}} L = \left( \nabla_{\vec{w}} L \odot Q \right) \odot \mathbf{1}_{\{ 0 < \vec{p} < 1 \}}$. 

\section{Theory}\label{sec:theory}

In this section we provide a thorough theoretical analysis of the capabilities of our models, using probabilistic methods and tools from convex random geometry. 
First, we argue that it is possible to recover Kaiming-He initialisation \cite{he}, making the parametrisation of this model comparable to that of arbitrary architectures from a functional perspective. 
Secondly, we explore the effect of the degree and size parameters $d, n \in \mathbb{N}$. 
Increasing $d$ provides performance benefits, at the cost of a less sparse matrix. 
\Cref{sec:d} is dedicated to explaining the performance benefits of $d$.

\begin{lemma}\label{lem:init}
Let the nonzero entries of the influence matrix  $Q$  be distributed as:

\[q_{i,j} \sim \mathcal{N}\left(0, \frac{6}{d n_l}\right).\]

Let $ p_j \sim \text{s-dist}[0,1], \; j = 1, \dots, n$, be independent and identically distributed (i.i.d.), where  \text{s-dist}[0,1]  is a symmetric distribution with support in [0, 1]. Define the vector  $\vec{w} = Q\vec{p}$ , where each component  $w_i$  is given by:
$w_i = \sum_{j = 1}^n p_j q_{i,j}.$
Then, for sufficiently large  d , the distribution of  $w_i$  converges in distribution to:
\[
w_i \xrightarrow[]{\mathcal{D}} \mathcal{N}\left(0, \mathbb{E}[p_j^2] \frac{6}{n_i} \right),
\]
which simplifies to Kaiming-He normal initialization with variance  $\sigma^2 = \frac{2}{n_l}$ (where $n_l$ is appropriately chosen for each $i$) in the case where  $p_j \sim U[0,1]$. In particular, the variance of a neuron with fan-in $n_\ell$ is then $\Var\left(\relu{\sum_i^{n_\ell} w_i} \right)=1$.
\end{lemma}
The proof can be found in the appendix.

\subsection{\texorpdfstring{Why large $d$ helps}{Why large d helps}}\label{sec:d}
In practice, $Q$ can be very large and difficult to store in the computer's memory. 
$d$ controls the sparsity of the matrix; picking a small $d$ maximises the advantages of sparse matrix computation, both in storage and inference. 
However, a very small $d$ can negatively impact the performance of training. 
In this section we give a few reasons why larger values of $d$ help.
In summary: 
\begin{enumerate}
\item Recall that $\mathcal{I}_i$ the set of non-zero entries in row $i$ of $Q$ with $|\mathcal{I}_i|=d$.
The probability that a weight $w_i = Q_i \vec p$ is zero is $\prod_{j\in \mathcal{I}_i} (1-p_j)$, which generally speaking, decreases with $d$.
\item For small values of $d$, a large proportion of columns ($e^{-d}$) are entirely zeros. 
This affects the rank of $Q$ and implies tat some entries of $\vec{p}$ become ineffective during training, harming the expressivity of the model. 
\item Given a matrix $Q$ for any row $i$, \ \[\max_{\vec p \in [0,1]^n} \E{ |Q_i \vec p| } = \Theta\left(\sqrt{\frac{d}{n_\ell}}\right) ,\]
where $n_\ell$ is the corresponding fan-in value.
We prove this in \Cref{pro:cherrypickD}.
 The possibility of large values for $\vec{w}$ leads to increased flexibility in the model, even at the cost of reducing the magnitude of other parameters. 
\end{enumerate}

\subsubsection{Increasing $d$ decreases the sparsity of $\vec{w}$}
\begin{lemma}
Let $\vec{z} \in \mathbb{R}^n$ be a random vector with $z_j \sim Bern(U(0,1))$, and $Q \in \mathbb{R}^{m\times n}$ be the influence matrix. 
Then the expected number of nonzero entries of $
\vec{w} = Q\vec{z}$ is: 
\[\mathbb{E}(\# \text{nonzero entries of } \vec{w}) = m \left(1-\frac{1}{2^d}\right).\]
\end{lemma}
\begin{proof}
The $j$-th element of $\vec{w}$ is given by 
$w_j = \sum_{k \in S_j} z_j, $
where $S_j$ is the set of indices corresponding to the entries of $j$-th row of $Q$, with $|S_j| = d$. 
Let $p^\star_k \sim U(0,1)$; then, the expected value of $P(z_k = 1) = p_k$. 
Moreover, $\mathbb{E}(p_k) = \frac{1}{2}$.
this entails that the probability that all $z_k$s in $S_j$ are zero is $\frac{1}{2^d}$. 
Clearly, 
$P(w_j \neq 0) = 1 - \frac{1}{2^d}.$
Summing across the $m$ rows returns the desired result.
\end{proof}

\subsubsection{Small $d$s lead to inexpressive $\vec{p}$}
\begin{lemma}\label{lem:columns}
The probability that exactly $k \geq 1$ columns of $Q$ are only zeros is given by: 
\[ P( k \text{ columns of } Q \text{ are empty}) = \frac{{{n}\choose{k}} {{n-k}\choose{d}}^m}{{{n}\choose{d}}^m}. \]
Moreover, the proportion of empty columns for large $m = n$ is $\approx e^{-d}$.
\end{lemma}
The proof is in Appendix \ref{app:combinatorics}. From this probability density function we already see that an increasing $d$ decreases the $({n-k\choose d}/{n\choose d})^m$ constant. 
Moreover, in Appendix \ref{app:combinatorics},we also show that the expected proportion of empty columns is approximately $e^{-d}$. 
For $d =1$, approximately $36\%$ of columns are empty, meaning that there are only $0.64n$ effective parameters. 
However, with $d = 10$, the proportion of empty columns is $\approx 0.000045$. 
Picking a $d$ that is not small significantly increases the expressivity of the model.  

\subsubsection{Increasing $d$ increases maximum magnitude of $w_i$}
\begin{proposition}\label{pro:cherrypickD}
   For any $Q$ drawn as described in \Cref{sec:model}, for any $i\in [m]$, we can find a vector $\vec p$ (that depends on $Q$), such that 
   \[\max_{\vec p \in [0,1]^n} \E{ |Q_i \vec p| } = \Theta\left(\sqrt{\frac{d}{n_\ell}}\right) ,\]
where $n_\ell$ is the corresponding fan-in value.
\end{proposition}
\begin{proof}
  Fix an arbitrary row $i$.
  Note that $\E{q_{i,j} | q_{i,j} > 0}= \sigma \sqrt{\frac{2}{\pi}}$ since this is equivalent to the expected value of the half-normal distribution. 
 By symmetry, $\E{q_{i,j} | q_{i,j} > 0}\leq -\sigma \sqrt{\frac{2}{\pi}}$.
 Consider $Q_i$.
  Assume w.l.o.g there there are more positive than negative entries $q_{i,j}, j\in [n]$ (otherwise, an analogous argument can be made).
  For every value $q_{i,j}$ in $Q_i$ that is positive, choose $p_j = 1$ and choose $p_j=0$ otherwise.
 Now for the upper bound note that there are at most $d$ positive values and hence $\max_{\vec p \in [0,1]^n} \E{ |Q_i \vec p| } \leq d \sigma \sqrt{\frac{2}{\pi}} $.
 Thus, $\max_{\vec p \in [0,1]^n} \E{ |Q_i \vec p| } \geq \frac{d}{2} \sigma \sqrt{\frac{2}{\pi}} $.
 Plugging in the value of $\sigma= \sqrt{\frac{6}{d n_\ell}}$ yields the claim.
\end{proof}

\subsection{\texorpdfstring{Compressing by decreasing $n$}{Compressing by decreasing n}}
While our approach reduces the communication complexity dramatically, it comes with other tradeoffs in computational costs. 
There are two reasons that we an incur a cost to accuracy. 

Firstly, decreasing $n$ leads to fewer effective parameters in the network. 
This means that weights in $\vec{w}$ propagate their gradients back to fewer $\vec{p}$, that have to aggregate across gradients of several weights. 
In particular, a simple calculation shows that: 
$\mathbb{E}(\text{expected nonzero entries of column } j \text{ in } Q ) = \frac{md}{n}.$
These can be interpreted as the numbers of parameters that are shared. 
We see that increasing $d$ and decreasing $n$ both lead to an increase in the expected number of weights affected by one entry of $\vec{p}$. 

Secondly, decreasing $n$ also affects the expressivity of the parameter space. 
When $n = m$, $Q$ is full rank (provided $d$ is sufficiently large) and the relationship between the weights can be represented by a random linear system $Q\vec{p} = \vec{w}$. 
However, whenever $n<m$, $\vec{p}$ is projected in a subspace with dimension at most $\dim(Q\vec{p}) = n$. 

While increasing the $d$ recovers some of the accuracy, as per \Cref{sec:d}, it comes with increased computational costs. 
Normally, a forward pass takes $\Theta(m)$ operations, one per parameter. Crucially, we work with sparse data structures.
In our case, we need to do $d$ multiplications per parameter to calculate the weight vector $\vec w$. Thus, the complexity becomes $\Theta(d m)$.
The extra computational costs come from:
\begin{itemize}
    \item Initialisation: generation of sparse matrix. 
    Algorithm creates $d$ entries per $m$ rows, $\mathcal{O}(md)$.
    \item Forward pass: a sparse CSI vector-matrix multiplication step in each epoch $\mathcal{O}(\text{nnz}(n) d) = \mathcal{O}(nd)$
    \item Backpropagation: an extra step in $\mathcal{O}(nd)$. 
\end{itemize}

In \Cref{sec:compression_Exp} we provide evidence of the tradeoff, whereas in \Cref{sec:federated_exp} we show that in the federated learning setting the performance costs of $n$ are mitigated. 
Overall, these costs are offset by the ability to compress the model, which can be very valueable in federated and mobile applications, and the generalisation benefits provided by the framework, which we discuss in the following section. 

\subsection{Sampling Parameters on Convex Shapes}\label{sec:convex}

\begin{figure}[ht!]
\includegraphics[width=0.35\textwidth]{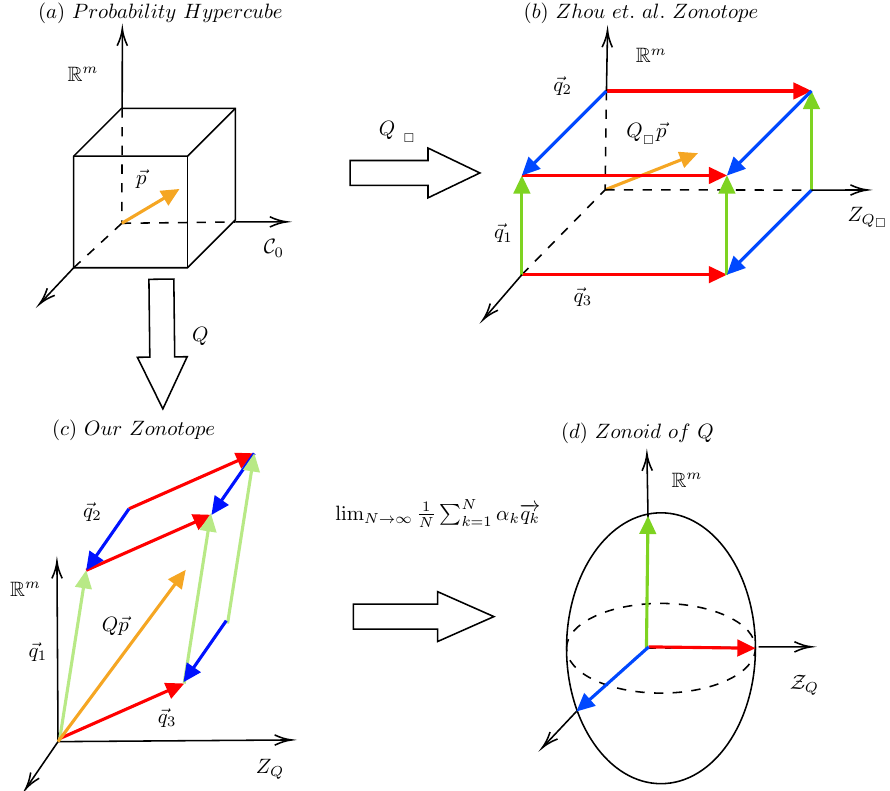}
\caption{We assume $\mathbb{R}^n=\mathbb{R}^m$ 
(a) hypercube realised by the possible values of vector $\vec{p}$,
(b) hyperrectangular zonotope generated by the choice of diagonal influence matrix, as in \cite{zhou},
(c) polytopal zonotope generated by our choice of influence matrix,
(d) zonoid of the zonotope generated from matrix $Q$, which yields an ellipsoid. }
\label{fig:zonotopes}
\Description{We assume $\mathbb{R}^n=\mathbb{R}^m$ 
(a) hypercube realised by the possible values of vector $\vec{p}$,
(b) hyperrectangular zonotope generated by the choice of diagonal influence matrix, as in \cite{zhou},
(c) polytopal zonotope generated by our choice of influence matrix,
(d) zonoid of the zonotope generated from matrix $Q$, which yields an ellipsoid. }
\end{figure}

It is well known that sampling acts as a regularizer and improves the generalisation capabilities of model \cite{srivastava2014dropout}; in this section we provide a novel theoretical bridge that aims to explain this phenomenon granularly. 
Indeed, we provide a characterisation of the generalisation capabilities of model through two measures of the probability distribution the model has converged to: the dimension and the volume of the space we explore via sampling.

Specially, we argue two claims. 
Firstly, that sampling by training improves generalisation, providing a novel link with random convex geometry. 
To do this we show that the training algorithm is sampling vertices of a geometric shape called a zonotope. 
We compute the volume of the zonotope explicitly, as an indicator of the volume of the explored space. 
Secondly, we show that federated learning further improves generalisation by increasing the dimensionality of the space of exploration over which we sample. 

Zonotopes are well-studied mathematical objects in convex geometry.
Their relevance is already known in deep learning, for instance in the context of describing the expressive power of ReLU networks \cite{zhang}. 
In this section, we formalise sampling by training through the formalism of subset search on random zonotopes.
Specifically, we notice that the influence matrix $Q$ induces a convex shape and characterise training-by-sampling as an exploration of the subset of vertices of the convex shape of interest. 
Through this intuition, we provide novel insight on the generalisation capabilities of these training methods. 
This may be the mechanics underpinning the success of approaches stemming from the Lottery Ticket Hypothesis \cite{frankle}, and sampling techniques such commonly used dropout \cite{srivastava2014dropout}. 

Importantly, we show in \Cref{sec:fedswithbenefits} that there are unique benefits to reap by sampled training in a federated learning setting. 
Averaging across $\vec{p}$s of all clients ensures that the dimension of the zonotope that is being explored remains large, even across runs. 

\subsubsection{Volume: Sampling on Random Zonotopes}
The choice of influence matrix $Q$ induces a random convex shape called a \textit{zonotope}. 
Sampling $Q \vec z$ is equivalent to exploring the vertices of this shape. 
In what follows, we make these statements precise and provide an intuition on how the different choices of parameters (the degree $d$ and size $n$ of the matrix) affect the search space. 
We refer the interested reader on the calculus of zonoids and their relation to zonotopes to \cite{mathis}.   
\begin{definition}[Random Zonotope]\label{def:zonotope}
Let $\mathcal{Q} = \{\vec{q}_1,\vec{q}_2,..., \vec{q}_n\}, q_j \sim N(0, \sigma^2), j = 1,..,n$ be a collection of random vectors. 
The random zonotope of $Q$, $Z_\mathcal{Q}$ is the random convex polytope defined by: 
\[ Z_\mathcal{Q} =\left\{ \sum_{j = 1}^n \alpha_j\vec{q_j}, \quad \alpha_j \in [0,1], j =1,..., n\right\}.\]
\end{definition}

A zonotope is a specific type of polytope with symmetries in its two dimensional faces \cite{mathis}. 
A zonoid is a convex shape, that is not necessary polytopal, that is linked to the expected realisation of a zonotope. 
In our case, the zonoid of $Q$ is given by: 
\begin{equation}\label{eq:zonoid}\mathcal{Z}_Q = \lim_{N\rightarrow \infty} \left\{\frac{1}{N} \sum_{k = 1}^N \alpha_k \vec{q_k}, \alpha_k \in [0,1], k= 1,..., N \right\}.\end{equation}   

For example, when $\vec{q_j}$s are standard normal, we have that the zonoid is $\frac{1}{\sqrt{2\pi}}B_n$, where $B_n$ is the unit ball in $\mathbb{R}^n$.
In \Cref{fig:zonotopes}, we provide a visual intuition of the referenced convex shapes and their relationships. 
This connection to zonoids allows us to compute explicitly the expected volume of a zonotope, which represents the size of region that we sample from in the network's parameter space. 

\begin{proposition}\label{prop:gaussian_zonotope}
Let $Q$ be the $n\times n$ influence matrix where each entry $q_{i,j} \sim N(0, 6/dn_i)$.
The volume of the zonotope (as in \cref{def:zonotope})
\[\mathbb{E}(\text{vol}_n({Z}_Q)) =\frac{n!\left({3}/{d}\right)^{{n/2}}}{\Gamma(1+n/2)}\prod_{i=1}^n \sqrt{\frac{1}{n_i}} .\]
\end{proposition}
\begin{proof}
By Vitale's formula \cite{vitale}, we have that: 
\[\mathbb{E}(\text{vol}(Z_Q)) = \mathbb{E}\left(\text{vol}_n\left(\sum_{j = 1}^n \vec{q}_j\right)\right) = n! \text{vol}_n (\mathcal{Z}_Q). \]

Using Proposition 5 \cite{mathis2022gaussian}, the zonoid (see \Cref{eq:zonoid}) for the matrix $\varepsilon$ with i.d.d. columns $\varepsilon_j \sim N(0,1)$ is: 
$Z_\varepsilon = \frac{1}{\sqrt{2\pi}} B_n,$
where $B_n$ is the $n$-dimensional ball in $\mathbb{R}^n$. 
From Lemma 2 (and as the author also states at the beginning of proof of Theorem 8) \cite{mathis2022gaussian}, it follows that, the zonoid $\mathcal{Z}_Q$ is a linear transformation of the Gaussian ball in $\mathbb{R}^n$, $Z_Q = B_n\Sigma/\sqrt{2\pi} = B_n \Sigma'$, where $ \Sigma^2 = \text{diag}((0, 6/dn_i )_{i = 1,..., n})$.
This is a dispersion ellipsoid; the volume of this object is well known to be: 
\[ \text{vol}(\mathcal{E}) = {\text{vol}(B_n)}\det(\Sigma') = \frac{\pi^{n/2}}{\Gamma(1+n/2) }{|\det(\Sigma')|},\]
where $B_n$ is the unit ball in $\mathbb{R}^n$ and $\Gamma$ is the gamma function.
The determinant of a square diagonal matrix is given by 
\[\det(\Sigma') =\prod_{i = 1}^n \sqrt{\frac{6}{2dn_i\pi}} = \left(\frac{3}{d\pi}\right)^{n/2} \prod^n_{i=1}\sqrt{\frac{1}{n_i}},\] returning the desired output. 
\end{proof}

This applies directly when $d=m=n$ in the influence matrix $Q$. 
Increasing the sparsity introduces correlation between the columns of $Q$, due to how the $d$ non zero elements are selected, while also introducing zeros in the matrix. 
Therefore we expect the volume to vary in practice; however, we take the theorem to be an indicative value of the size of the ellipsoid in which we sample during training for $d<n$.


\subsubsection{Dimension: The Benefits of Federated Learning}\label{sec:fedswithbenefits}
Training-by-sampling in \cite{zhou} is equivalent to exploring the vertices of a rectangular zonotope, which is a hyperrectangle. 
The distribution on the vertices is updated until one finds an subset of the hyperrectangle where good performance is achieved. 
In \cite{zhou}, the influence matrix is $Q_{\square} = \text{diag}(\vec{q}), \vec{q} \sim N(0, \Sigma)$, where $\Sigma_{i,i} = \sigma_j^2, \Sigma_{i,j} = 0, j \neq i$. 
The supermasks are parametrised by vectors $\vec{p}, p_j \in [0,1]$, from which we sample: $z_j \sim \text{Bern}(p_j)$, so that $\E{w_{\square}} = Q_{\square}p$. 
In otherwords, $\vec{z}$ are sampled vertices of a hypercube, which are rescaled by randomly initialised weights $Q_{\square}$, kept constant in training. 

The probability distribution on the vertices of the unit hypercube is given by: 
$P(\vec{z} = \vec{z}') = \prod_{j = 1}^m z_j'p_j + (1-z_j')(1-p_j).$
This leads to: 
$P(\vec{w} = \vec{w}'| Q_{\square} = Q_{\square}') = \prod_{j = 1}^m q'_{j,j} \left(z_j'p_j + (1-z_j')(1-p_j)\right).$

\begin{definition}[$\tau$-Hypercube]
Let $\mathcal{C}_\tau \subset \mathbb{R}^n$ be a lower dimensional hypercube, given by:
\[\mathcal{C}_\tau = \left\{ \vec{\alpha}\odot \vec{\mathbb{I}}({\vec{p}}) ,  \sum_{j = 1}^n \alpha_j \mathbb{I}_{\{ \tau \leq p_j \leq 1-\tau\}},  \tau \in [0,0.5], \alpha_j \in [0,1]\right\},\] 
where $\vec{\mathbb{I}}({\vec{p}}_j) = \mathbb{I}_{\{\tau \leq p_j \1-\tau \}}$, a vector with indicator variables which which are $1$ if $p_j \in [\tau, 1-\tau]$ and $0$ otherwise. 
\end{definition}

Note that the expected dimension of this hypercube is at initialisation is: 
\[\mathbb{E}(\dim(\mathcal{C}_\tau)) = \mathbb{E}\left(\sum_{j = 1}^n \mathbb{I}_{\{\tau \leq p_j \leq 1-\tau\}}\right) = {n}(1-2\tau).\]

\begin{proposition}[Benefits of Federated Learning]\label{prop:fed_learning}

Let $\vec{p_1}, \vec{p_2}, ..., \vec{p_C}$ for a number of clients $C \in \mathbb{N}$, be a collection of vectors, with $\tau$ hypercubes given by
$\mathcal{C}^{\vec{p_1}}_\tau, ..., \mathcal{C}^{\vec{p_C}}_\tau$. 
Then the hypercube of the average $\vec{\bar{p}}$ has dimension: 
\[ dim(\mathcal{C}^{\vec{\bar{p}}}_\tau) \geq \frac{1}{C} \sum_{ c = 1}^C dim(\mathcal{C}^{\vec{p}_{c}}) \] 
\end{proposition}
\begin{proof}
Let $\#_\tau: \mathbb{R}^n \rightarrow \mathbb{N}$ be the function that counts the number of non-trivial $p_i$ in $\vec{p}$. 
This is a convex function; therefore, by Jensen's inequality: 
$\frac{1}{C}\sum_{c=1}^C \#_\tau(p_{i,c}) \geq \#_\tau\left(\frac{1}{C}\sum_{c=1}^C p_{i,c} \right). $
In other words, the number of non-trivial $p_i$ is larger for averaged $p_i$s. 
This directly entails the inequality since $\#_\tau = n - dim(\mathcal{C_\tau})$. 
\end{proof}

The random zonotope of $Q$ can be expressed as a vector product of sets: 
$Z_Q = Q C_0 = \left\{Q \alpha, \alpha \in C_0 \right\}.$
Together, Proposition \ref{prop:gaussian_zonotope} and Proposition \ref{prop:fed_learning} provide insight on what happens during training-by-sampling. 
At initialisation, we generate an area of exploration via the zonotope of $Q$, $Z_Q$. 
Here, we sample the vertices of $Z_Q$, with probabilities determined by $\vec{p}$ and compute the loss and gradients with our optimisers. 
Then, during training, the entries of $\vec{p}$ become extreme: closer to 0 or 1. 
In turn, this decreases the dimension set of the that we are likely to sample in, which becomes a set of dimension $\dim(C_\tau)$. 
Our algorithm ensures that we find a set of solutions $S = Q C_\tau$ where there is overall good performance. 
We show experimentally in \Cref{sec:sensitivity_exp} that this leads to drastically improved generalisation compared to training without sampling. 

\section{Experiments}\label{sec:experiments}
Our experiments complement the theoretical exposition in \Cref{sec:theory}, confirming the mathematical claims and showcasing the models' capabilities in the federated setting. 
The goal of these experiments is to evaluate the performance of \localmodel and \binmodel, as we compress the number of parameters, forcing the weight sharing scheme defined by the  matrix $Q$. 
The experiments \Cref{sec:compression_Exp} \Cref{sec:federated_exp} provide evidence that the tradeoff between compression and accuracy brought by a choice of small $n$ can be minimised by using sufficiently large $d$ and deploying the \binmodel in the federated setting; \Cref{sec:sensitivity_exp} and \Cref{sec:zhou_exp} support the claim that \localmodel alone can provide benefits in generalisation and performance, even compared to \cite{zhou}.

\paragraph{Experiment Summary}
We run four experimental evaluations. 
First we show the tradeoff between compression factor and performance across different $d$s in \Cref{sec:compression_Exp}. 
As predicted from the theory, increasing $d$s marginally above $1$ increases the accuracy. 
In this experiment we explore the effect of compression. 

Second, we test our framework in the federated learning setting applying compression. 
We run three simulations showing that a factor $32$x compression on architecture size (and $1024$x compression against the standard case). 
This provides evidence of the benefits of sampling in presented in \Cref{sec:convex}. 

Third, we run our experiments to showcase the generalisation properties of our approach. 
We do this by perturbing $\vec{p}$ across the dimensions where $\tau <p_j <1-\tau$ by a Gaussian impulse. 
We provide evidence of robustness of the \localmodel's accuracy to strong perturbations $\vec{p}$ when compared to training the expected network (i.e. training without sampling). 
As predicted in \Cref{sec:convex}, perturbing the non-trivial entries of $\vec{p}$ of the model leads to minimal effect on the network performance, supporting the view that our framework finds regions of strong generalisation. 

Finally, we show evidence that our model is superior to Zhou et. al. \cite{zhou}.
Over very few runs, we show clear advantages in using \localmodel across all choices of $d$. 
We believe this is due to the benefits induced by a larger $d$ \Cref{sec:d} and the generalisation benefits described in \Cref{sec:convex}. 

\paragraph{Experimental Constant}
We report the experimental choices that were present in all experiments to avoid repetition. 
More detail for each experiments is reported in each section. 
The evaluation of basic method was run on a machine with GPU RTX3080 with 12GBs of VRAM, with AMD Ryzen Threadripper 3960X 24-Core CPU and RAM 256GBs.
In all our experiments we use the MNIST dataset and use the framework described in \Cref{sec:model}.
We run each training round for 100 epochs with early stopping, using 10 epochs of patience and a delta of $10^{-4}$. 
All our training is run using Adam optimizer, with momentum 0.9 and varying learning rate. 
Everywhere, we use the standard MNIST data splits with batches of size 128. 
The model’s parameter initialization followed a uniform distribution on $\vec{p}$ and we choose $q_{i,j}$ to be distributed as in \Cref{lem:init}, to recover Kaiming-He initialisation.
The code is reproducible and shared at \textit{{github.anon.com}}, using PyTorch.

We use one of two architectures in each of our experiments: \smallarch or \mnistfc. 
Our compression experiments \Cref{sec:compression_Exp} and sensitivity \Cref{sec:sensitivity_exp} experiments we use \smallarch, which is a feedforward neural networks with two hidden layers and twenty neurons per layer.
The \mnistfc is exactly as the one in Zhou, two hidden layers with three hundred and one hundred neurons respectively.
This architecture is used both in the federated setting \Cref{sec:federated_exp} and in comparison with \citep{zhou}, in \Cref{sec:zhou_exp}.


\subsection{\texorpdfstring{Varying the Compression Factor $m/n$}{Varying the Compression Factor m/n}}\label{sec:compression_Exp}
\paragraph{Setup.}
In this experiment we vary the compression factor in the $\localmodel$.\footnote{Running the experiment in the $\localmodel$ allows us to isolate that impact of the compression rate, with no impact coming from the averaging binary numbers (recall that in $\binmodel$ the clients sends binary numbers).} 
The goal of this experiment is to see the impact $n$ plays on the accuracy.
We train architectures with for 5 levels of $d = 1, 5, 10, 50, 100$ and with compression factor 11 levels of $m/n = 2^{i}, i = 0,...,10$ reported in \Cref{fig:compression}.
For each architecture, we run 5 random seeds (seed = 0,..,4) with learning rate 0.001. 
After a training round, we sample 100 networks and compute mean sampled accuracy. 
We report the average of the sampled accuracy and the standard deviation in the \Cref{fig:compression}.

\paragraph{Analysis.}
We see in \Cref{fig:compression} the trade-off between the compression factor and the accuracy.
It appears that increasing $d$ beyond $5$ makes little difference, which tracks with the mathematical of impact we discuss in \Cref{sec:d}.
For $d=1$ the results are consistently worse. 
The expected accuracy ($\vec w^*= Q \vec p^*$, where $\vec p^*$ is the final probability vector) overall is almost the same as the sampled accuracy ($\vec w^*= Q \vec z^*$, with $\vec z^* \sim \vec p^*$). 
The accuracy seems to follow roughly inverted logarithmic trend, i.e., doubling the compression factor leads to a constant drop. 
In particular, for $d=5$ we see that a $m/n=4$ results in a drop of about on average $5$ with the standard deviation being of about the same order of magnitude. 
The values of the experiment can be found in \Cref{app:extended_compress}. 

\begin{figure}[ht!]
\includegraphics[width=0.3\textwidth]{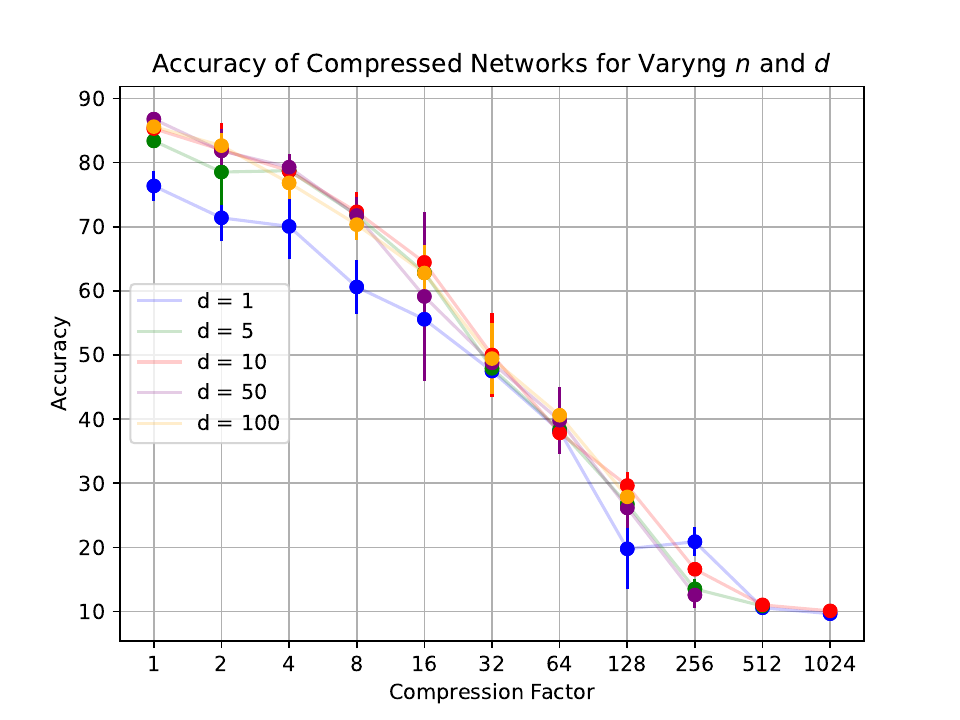}
\caption{Trade-off between compression and accuracy \smallarch in \localmodel for varying levels of $d$.}\label{fig:compression}
\Description{Trade-off between compression and accuracy \smallarch in \localmodel for varying levels of $d$.}
\end{figure}

\subsection{Federated Learning Experiments } \label{sec:federated_exp} \FloatBarrier
\paragraph{Setup}
In the federated learning setup, we ran three simulations with 10 clients and one server. 
Each client was trained over a total of 100 rounds. 
The data was partitioned with a random split. 
In this experiment, we tested the \mnistfc model with training-by-sampling, measuring the accuracy on the expected network. 
The model was initialized with $n = m/i, i = 1,8,32$, where $m = 266610$  and a $d = 10$, with $\vec{p}$ initialized uniformly and learning rate is 0.1, random seed is 1. 
We compute the mean sampled accuracy at each round, together with the standard deviation out of 100 sampled networks. 

\paragraph{Analysis}
Results are displayed in \Cref{fig:isik_compare}. 
Bench-marking against $m/n = 1$, we see that our performance gets virtually no loss in performance (.22\%) for a 8 fold reduction in parameters in the $m/n = 8$ experiment. 
Moreover, in the $m/n = 32$ compression experiment, we recover (2.55\%). 
The metrics are summarised in \Cref{tab:isik_compare} and compared to \cite{isik}. 

\begin{figure}[ht!] \label{fig:isik_compare}
\includegraphics[width=0.5\textwidth]{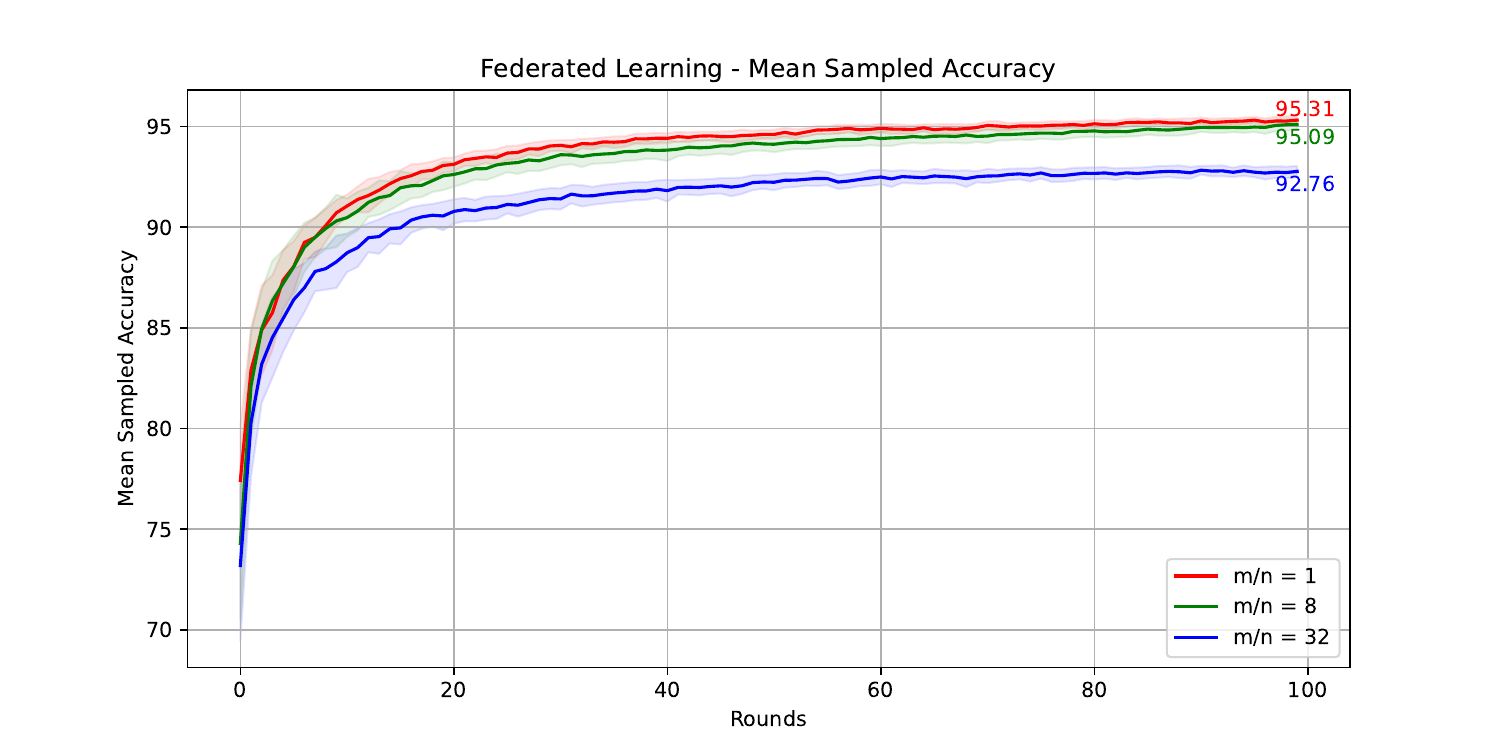}
\caption{Results of training \binmodel in the federated learning framework with varying levels of $d$.}
\Description{Results of training \binmodel in the federated learning framework with varying levels of $d$.}
\label{fig:isik}
\end{figure}

\FloatBarrier

\begin{table}\label{tab:isik_compare}
\begin{tabular}{|c|c|c|c|}
\hline
 & client savings & server savings & test accuracy \\ 
\hline
\cite{isik} & 33.69 (*)  & 1.05 (*) &0.99 \\ 
\hline
\textbf{[us]} $m/n=8$ & 256 & 8 & 0.95 \\ 
\hline
\textbf{[us]} $m/n=32$ & 1024 & 32 & 0.93 \\ 
\hline
\end{tabular}
\caption{All savings are measure by what factor the communication cost decreases per round in comparison to the naive protocol that sends all $m$ parameters as floats. For example, just sends a bit rather than a float results in 32. (*) The bit-rate they achieved was about 0.95 due to artethmetic compression. Note that they use a larger architecture (about 7 times larger).}
\end{table}
Comparing the 1,933,258 parameter ConvNet architecture in Isik et. al. \cite{isik}, we use a 266,610 parameter feedforward (about 7 times smaller) architecture. The clusters we have access to were unable to run their architecture, which is why we only have a test accuracy of $0.95$ (instead of their $0.99$) even without compression. We believe that our results on their architecture would result in even higher accuracies and with higher compression factors  due to their model being much more over-parameterised.

\subsection{Generalisation Properties via Parameter Sensitivity}\label{sec:sensitivity_exp}
\paragraph{Setup}
This experiments aims to verify the claims in \Cref{sec:convex} pertaining to the generalisation capabilities. 
We claimed that the network finds a convex subspace of $\mathbb{R}^m$ where there is good performance. 
To validate our claims, we train \localmodel under two regimes: either via sampling or via Adam on $\vec{w} = Q\vec{p}^{\star}$.
For each we compute performance by perturbing the final learned vector $\vec{p}$ by an $\epsilon_\tau$, where $\epsilon_{\tau,i} \sim N(0,1)$ whenever $\tau\leq p_j\leq 1-\tau$ and is 0 otherwise. 
Performance is computed across 10 perturbations and averaged. 
We compute mean sampled accuracy (across 10 sampled networks) and accuracy for the respective methods. 
We then compute the average sensitivity (change in performance divided by initial performance) and average deviatiation (change in performance divided by the L2 norm of $\eps_\tau$).
We report the standard deviation across the 10 measurements. 

\paragraph{Analysis}
We display the results in \Cref{tab:hypercube_sensitivity}. 
Overall, the clear trend is that the sampled network both performs better and witnesses average sensitivity and deviation smaller by two orders of magnitude, signalling the generalisation benefits of training-by-sampling. 
Moreover, even when $\tau = 0.5$ (and therefore all values $p_j$ are perturbed), we see a gap in robustness between training-by-sampling and traditional training of the expected network: an 11\% drop versus a 62\% drop in performance.

\section{Conclusions}
We introduced a new framework capable of learning on a compressed parameter space while maintaining remarkable accuracy in the federated learning setting. 
Importantly, we introduced a novel link between random convex geometry and training-by-sampling, which we expect will lead to new insight into probabilistic training of networks. 
Using zonotope calculus, we  characterize the sample space of our compressed network and we establish properties regarding the volume and dimensions of the sample space. 
We hope  this new perspective on generalisation and the \textit{density of good solutions} in the parameter space. 

While this work focuses mostly on the \textit{silo} federated learning setting (where one central server receives and spreads messages from clients), we envision future work to focus on a distributed setting, without a central server, testing the performance of \binmodel where the communication between clients follows arbitrary graph patterns. 
On the other hand, further theoretical analysis might give insights into the relationship between the weight sharing literature, for instance, a direct comparison with works such as \cite{chen2015compressing}. 

Finally, in \Cref{sec:sensitivity_exp} we have shown that some entries of $\vec{p}$ are not influential in the forward pass of the model. 
This suggests that further compression of $(Q,\vec{p})$ can be achieved, without loss of accuracy. 
Specifically, we can remove the columns of $Q$ related to trivial $\vec{p}$s, and reduce the rows of $Q$ when weights are summed to 0. 
We conjecture this will decrease further the communication cost in the federated learning setting.

\newpage
 \bibliographystyle{ACM-Reference-Format}
   \bibliography{biblio}

\appendix

\section{Initialisation}\label{app:initialisation}
In this section we study the impact of different initialization.
Concretely we study how the network performance changes if we train the values $\vec p$ directly, without sampling, i.e., we compute the gradient w.r.t. to $\vec p$ rather than not $\vec z$. 
We run this experiment on the \mnistfc, by training the $Q\vec{p}$ traditionally (without sampling). 
We use a learning rate of 0.01 and change initialisation according to the parameters of a beta distribution. 
We take the average over 3 random seeds. 
We display the results in \Cref{fig:integralitygap}. 

The blue line at the top represents the test accuracy of the expected network, given by $\vec{w} = Q\vec{p}$ (without sampling). 
We see that whenever we sample a $\vec{z} \sim \text{Bern}(\vec{p})$, the performance of the network collapses, contrary to what is witness when the neural network is trained via sampling, where the difference between sampled and expected network performance is small.
We refer to this difference as the \textit{integrality gap}. 

Changing initialisation strongly affects the behaviour of the model. 
Small beta parameters imply that the distribution of $\vec{p}$ has high density around $0$ and $1$. 
This artificially decreases the decreases the integrality gap, making the red and blue performance curves closer to each other. 
As we increase the beta parameters, the two curves diverge. 
We conjecture that this is the case because the parameters stay close to their initial distribution when training without sampling: this entails that the difference between $\vec{z}$ and $\vec{p}$ is small at initialisation and remains small throughout training whenever then beta parameters are around 1/10.
We also see that for small beta parameters the variance and max-min gap of the performance of sampled networks (from 100 samples) is minimised. 
Again, this is because the extreme $\vec{p}$s force the $\vec{z}$s to specific values (either 0 or 1) with high probability. 

We also plot the curve of the discretized network. 
This is a network with $\vec{p}_{\circ}$, where $p_{\circ,j} = \text{argmin}_{z = 0,1} |p_j - z|$. 
Interestingly, the performance of the discretized network is enveloped around the curve of the expected and sampled networks for small beta parameters.
However, as we increase the distribution's parameters, we find that the discretised network performs worse than the mean sampled network. 
This is because the values of $\vec{p}$ are increasingly distributed around 1/2 and the discretization moves the network far from the learned parameters.  

\begin{figure}[ht!]
\makebox[0pt][l]{\hspace{-5.5cm}
\includegraphics[width=0.6\textwidth]{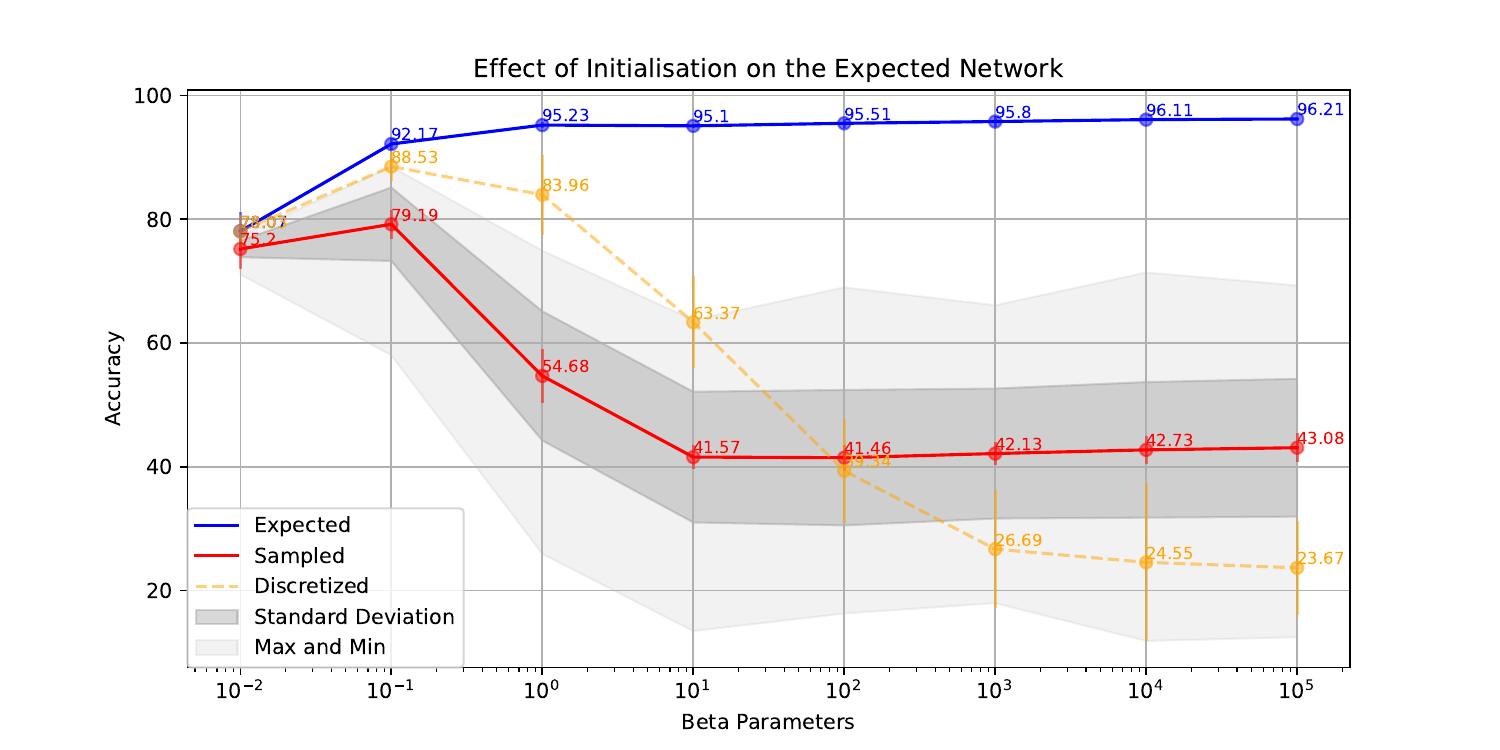}}
\caption{In this figure we study the impact training-by-sampling has. 
The figures shows that if we only train the $\vec p$ directly and then sample a network in the end, it is not robust.
    However, selecting initialization that have abundand extreme values, decreases the integrality gap.}
\Description{In this figure we study the impact training-by-sampling has. 
The figures shows that if we only train the $\vec p$ directly and then sample a network in the end, it is not robust.However, selecting initialization that have abundand extreme values, decreases the integrality gap.}
\label{fig:integralitygap}
\end{figure}

\section{Extended Results: Section~\ref{sec:compression_Exp}} \label{app:extended_compress}

When you perturb all values of $\vec{p}$ indiscriminately ($\tau = 0.5$), performance decreases marginally for sampled networks, but collapses for networks trained regularly.



\subsection{Comparison with Zhou et. al. \cite{zhou}}\label{sec:zhou_exp}
\paragraph{Setup}
For our comparison with Zhou et. al \cite{zhou} we aimed to contrast our performance of $\localmodel$ against the aforementioned approach by varying the degree of $Q$ in powers of 2 ($d = 2, 4, 16, 256$). 
We run each experiment on 5 random seeds, training the model with a learning rate of 0.001.
We sampled on 100 networks at the end of training and capture the best performance (as in \cite{zhou}, providing the standard deviation across the random seeds. 

\paragraph{Analysis}
This setup aimed to evaluate the effectiveness of \localmodel under conditions similar to those used in Zhou’s work, providing a basis for direct comparison.
\Cref{fig:zhou_compare} provides evidence that our method easily performs better than the supermask approach in \cite{zhou}. 
Interestingly, we verify the benefits in increasing to very large values ($d=256$), beyond avoiding small values of $d$. 
We conjecture this is due to the reasons provided in \Cref{sec:d}. 

\begin{figure}[ht!] \label{fig:zhou_compare}
\includegraphics[width=0.4\textwidth]{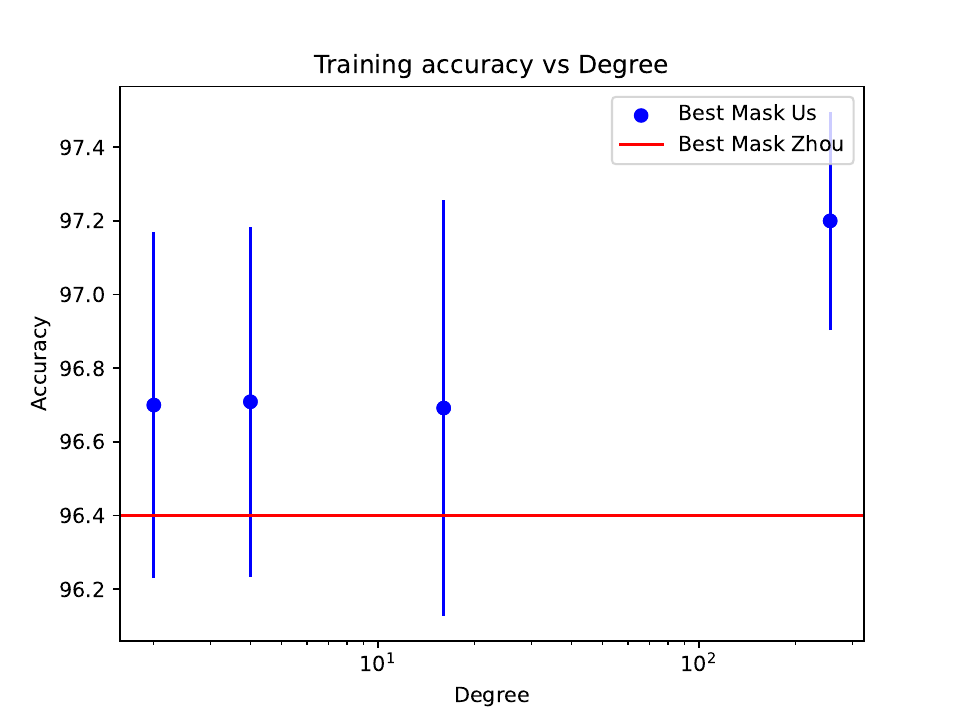}
\caption{Comparison between our method for different degrees $d$ and Zhou et. al. \cite{zhou}; best mask refers to the best sampled $\vec{z}$. }
\Description{Comparison between our method for different degrees $d$ and Zhou et. al. \cite{zhou}; best mask refers to the best sampled $\vec{z}$. }
\label{fig:zhou}
\end{figure}

\section{Missing Proofs}\label{app:combinatorics}
\begin{proof}[Proof of \Cref{lem:columns}]
(Recall that $k\leq d \leq n$). 
The number of possible ways to choose $k\geq 1$ columns is ${n}\choose{k}$. 
Once these have been chosen there are ${n-d}\choose{k}$ ways to fill the row of the column without selecting the $k$ entries that need to empty. 
This needs to be repeated $m$ time, for each row, so that: 
\[ \#(\text{Number of acceptable configurations of } Q )= {{n}\choose{k}} {{n-k}\choose{d}}^m.\]

The total number of ways to pick $d$ entries out of $m$ in $m$ rows is: 
\[ \#(\text{Total configurations of } Q) = {{n}\choose{d}}^m.\]
This entails that: 
$P( k \text{ columns of } Q \text{ are empty}) = \frac{{{n}\choose{k}} {{n-k}\choose{d}}^m}{{{n}\choose{d}}^m}. $

We proceed to show that the number of zero columns in $\approx e^{-d}$. 
By combinatorial argument, for each row, the probability that a specific column is not selected is:

$
P(\text{column $j$ not selected in one row}) ={\dbinom{n - 1}{d}}{\dbinom{n}{d}}^{-1}. 
$

\[
{\dbinom{n - 1}{d}}{\dbinom{n}{d}}^{-1} = \frac{(n - 1)! (n - d)!}{n! (n - 1 - d)!} = \frac{n - d}{n}.
\]

Since each row is independent, the probability that the column is empty across all  $m$  rows is:

\[P(\text{column $j$ is empty}) = \left( \frac{n - d}{n} \right)^m\]

An approximation for the expected number of empty columns (assuming independence of columns) is given by: 
\[
\mathbb{E}(\#(\text{empty columns})) \approx n \left( 1 - \dfrac{d}{n} \right)^{m}\,
\]
which, for large $n = m \gg d$ reduces to:
$
\mathbb{E}(\#(\text{empty columns})) \approx m e^{-d}.
$
\end{proof}
We also have that: 
\begin{align*}
    \mathbb{E}(\text{expected nonzero columns of }Q ) 
    &= \sum_{i = 1}^{n}k \frac{ {n\choose k} {n-k \choose d}^m}{{n \choose d}^m} \\
    & = \frac{1}{{n \choose d}^m} \sum_{i = 1}^d k {n\choose k} {n-k \choose d}^m.
\end{align*}
Whereas, the probability that there are no non-zero columns is: 
\begin{align*}
    P(Q \text{ has no nonzero columns}) 
    &= 1 - \sum_{i = 1}^d\frac{{n\choose k}}{{n \choose d}^m} {n-k \choose d}^m \\
    &= 1 -{n \choose d}^{-m} \sum_{i = 1}^d{n\choose k} {n-k \choose d}^m. \\
\end{align*}

In general, these quantities are really small: most of the genes are effective populated.
Moreover, notice that expected number of nonzero elements for a row is: \\
$\mathbb{E}(\text{expected nonzero entries of column } j \text{ in } Q ) = \frac{md}{n}.$
This quantity can be interpreted as the expected number of weights influenced by a gene.
Smaller $n$ implies a gene has a greater influence on the weights. 
Larger $d$ also makes all genes affect a larger number of weights.

\begin{proof}[Proof of \Cref{lem:init}]
Since \( p_j \) and \( q_{i,j} \):
$
\mathbb{E}[p_j q_{i,j}] = \mathbb{E}[p_j] \cdot \mathbb{E}[q_{i,j}] = 0.
$
Thus, the expected value of \( w_i \) is:
\[
\mathbb{E}[w_i] = \mathbb{E}\left[\sum_{j=1}^n p_j q_{i,j}\right] = \sum_{j=1}^n \mathbb{E}[p_j q_{i,j}] = 0.
\]
Next, we calculate the variance of \( w_i \). Since the \( q_{i,j} \)'s are independent normal random variables, the variance of \( w_i \) is the sum of the variances of the individual terms \( p_j q_{i,j} \):
$
\text{Var}(w_i) = \sum_{j=1}^n \text{Var}(p_j q_{i,j}).
$
Using the fact that $p_j$ and $q_{i,j}$ are independent, we can write:
$
\text{Var}(p_j q_{i,j}) = \mathbb{E}[p_j^2] \cdot \text{Var}(q_{i,j}) = \mathbb{E}[p_j^2] \cdot \frac{6}{d n_i}.
$
Thus, the total variance of \( w_i \) is:
$
\text{Var}(w_i) = \sum_{j=1}^n \mathbb{E}[p_j^2] \cdot \frac{6}{d n_l} = d \cdot \mathbb{E}[p_j^2] \cdot \frac{6}{d n_i}.
$
By the Central Limit Theorem, since $w_i$  is the sum of $ d $ independent random variables, and each $ p_j q_{i,j} $ has mean 0 and finite variance, for large $d$, the distribution of $w_i$ converges to a normal distribution:
\[
w_i \xrightarrow[]{\mathcal{D}} \mathcal{N}\left(0, \mathbb{E}[p_j^2] \cdot \frac{6}{n_i} \right).
\]

If \( p_j \sim U[0,1] \), then \( \mathbb{E}[p_j^2] \) is the second moment of the uniform distribution on \([0, 1]\). The second moment of \( U[0,1] \) is:
\[
\mathbb{E}[p_j^2] = \int_0^1 p_j^2 \, dp_j = \frac{1}{3}, 
\]
which yields Kaiming-He normal distribution. 
\end{proof}

\begin{table*}[htbp]
\centering
\caption{Mean Sampled Accuracy for Different $d$s and Compression Factors $m/n$.}
\label{tab:mean_sampled_accuracy}
\begin{tabular}{c|cccccc}
\hline
Weight degree& \multicolumn{6}{c}{Compression Factor} \\
$d$ & 1 & 2 & 4 & 8 & 16 & 32 \\
\hline
100 & $85.60 \pm 2.16$ & $82.63 \pm 2.43$ & $76.83 \pm 2.93$ & $70.33 \pm 2.98$ & $62.78 \pm 3.29$ & $49.43 \pm 3.32$ \\
50  & $86.77 \pm 1.76$ & $81.83 \pm 2.30$ & $79.28 \pm 2.76$ & $71.75 \pm 2.94$ & $59.12 \pm 3.71$ & $48.76 \pm 3.59$ \\
10  & $85.29 \pm 2.20$ & $81.99 \pm 2.50$ & $78.70 \pm 2.56$ & $72.31 \pm 2.90$ & $64.43 \pm 2.65$ & $49.99 \pm 2.99$ \\
5   & $83.37 \pm 2.42$ & $78.52 \pm 3.40$ & $78.73 \pm 2.81$ & $71.80 \pm 3.17$ & $62.85 \pm 3.02$ & $47.90 \pm 3.33$ \\
1   & $76.35 \pm 3.21$ & $71.37 \pm 4.25$ & $70.05 \pm 3.55$ & $60.60 \pm 3.49$ & $55.56 \pm 3.54$ & $47.48\pm 3.00$ \\
\hline
\end{tabular}
\caption{This plot shows our accuracy with varying compression rate $n/m$. A rate of $1$ means no compression and a rate of $1/r$ compresses the model by a factor of $r$.}
\label{table:compresion}
\end{table*}

\begin{table*}[htbp]
\centering
\caption{Table presenting random sampling in the $C_\tau$ hypercubes for networks trained via sampling versus regular training.}
\label{tab:hypercube_sensitivity}
\begin{tabular}{c|cc|cc|cc}
\hline
$\tau$ & \multicolumn{2}{c|}{Average Accuracy} & \multicolumn{2}{c|}{Average Sensitivity} & \multicolumn{2}{c}{Average Deviation} \\
 & Regular & Sampled & Regular & Sampled & Regular & Sampled \\
\hline
\hline
$0.01$ & $82.85 \pm 2.82$ & $90.29 \pm 0.64$ & $0.13 \pm 0.03$ & $(3.52 \pm 1.61) \times 10^{-3}$ & $0.09 \pm 0.02$ & $(2.50 \pm 1.13) \times 10^{-3}$ \\
$0.10$ & $81.43 \pm 4.01$ &$ 89.67 \pm 0.90$ & $0.14 \pm 0.04$ & $(3.04 \pm 0.93) \times 10^{-3}$ & $0.10 \pm 0.03$ & $(2.13 \pm 0.65) \times 10^{-3}$ \\
$0.20$ & $82.42 \pm 3.77$ & $89.64 \pm 1.12$ & $0.13 \pm 0.04$ & $(2.17 \pm 0.92) \times 10^{-3}$ & $0.10 \pm 0.03$ & $(1.52 \pm 0.63) \times 10^{-3}$ \\
$0.50$ & $20.07 \pm 0.99$ & $78.80 \pm 3.28$ & $0.78 \pm 0.01$ & $0.17 \pm 0.03$ & $0.55 \pm 0.01$ & $0.12 \pm 0.02$ \\
\hline
\hline
\end{tabular}
\end{table*}

\end{document}